\documentclass{article}
\usepackage[preprint]{neurips_2025}
\usepackage[utf8]{inputenc}
\usepackage{amsthm,amsfonts,amsmath,amssymb,xcolor,hyperref,enumerate,parskip}
\usepackage{aligned-overset}
\usepackage{wrapfig}

\usepackage{algpseudocode}
\usepackage{algorithm}
\usepackage{url}

\usepackage{caption,subcaption}
\UseRawInputEncoding
\newtheorem{theorem}{Theorem}[section]
\newtheorem{definition}{Definition}
\newtheorem{proposition}[theorem]{Proposition}
\newtheorem{corollary}[theorem]{Corollary}

\newtheorem{assumption}{Assumption}

\newtheorem{lemma}[theorem]{Lemma}

\makeatletter
\def\thm@space@setup{%
  \thm@preskip=\parskip \thm@postskip=0pt
}
\makeatother

\newcommand{\CLSI}[0]{C_{\mathsf{LSI}}}

\allowdisplaybreaks

\title{Private Continuous-Time Synthetic Trajectory Generation via Mean-Field Langevin Dynamics}
\author{
  Anming Gu\thanks{Part of this work was completed while AG was affiliated with Boston University.} \\
  UT Austin \\
  \texttt{anminggu@utexas.edu} \\
  \And
  Edward Chien \\
  Boston University \\
  \texttt{edchien@bu.edu} \\
  \AND
  Kristjan Greenewald \\
  MIT-IBM Watson AI Lab; IBM Research \\
\texttt{kristjan.h.greenewald@ibm.com} \\
}

\begin{document}

\maketitle

\begin{abstract}
We provide an algorithm to privately generate continuous-time data (e.g. marginals from stochastic differential equations), which has applications in highly sensitive domains involving time-series data such as healthcare. We leverage the connections between trajectory inference and continuous-time synthetic data generation, along with a computational method based on mean-field Langevin dynamics. As discretized mean-field Langevin dynamics and noisy particle gradient descent are equivalent, DP results for noisy SGD can be applied to our setting. We provide experiments that generate realistic trajectories
on a synthesized variation of hand-drawn MNIST data while maintaining meaningful privacy guarantees. Crucially, our method has strong utility guarantees under the setting where each person contributes data for \emph{only one time point}, while prior methods require each person to contribute their \emph{entire temporal trajectory}--directly improving the privacy characteristics by construction.
\end{abstract}

\section{Introduction}
In the era of modern machine learning, privacy has become increasingly important for users and the industry serving them. Differential privacy (DP) \cite{dwork2006dp} has become the de-facto standard for guiding methods of preserving user privacy. As DP (informally) works by bounding how much a single user can influence the output of a computation, it offers a \emph{worst-case} guarantee that composes gracefully across multiple analyses. As a result, it has seen widespread adoption beyond its initiation as a theoretical framework, finding extensive applications in industry \cite{apple_dp,ding2017microsoft,rogers2020linkedin} and government organizations \cite{abowd2018us,hod2024dp_israel}.  

Many approaches to achieve DP involve releasing noisy answers to a series of computation; however, the variance increases with the number of queries \cite{dwork2014algorithmic}, which degrades their utility. Furthermore, these approaches often only give utility guarantees for a fixed set of queries and not more complicated machine learning tasks such as classification. Thus, a promising approach for preserving privacy is \emph{synthetic data generation}, where the goal is to create a high-quality copy of a dataset under the constraints of differential privacy. Any subsequent tasks performed on this dataset are guaranteed to remain private due to the post-processing property of differential privacy. 

    Much work on synthetic data generation focuses on the tabular setting, where each row of the dataset contains an individual's records. On the other hand, trajectory (or time-series) data can be thought of as a concatenation of many tables at multiple time periods. Techniques to provide privacy for tabular synthetic data generation often do not work well in the time-series setting: since many of these methods leverage matching marginals in the binned space, the inherently high dimension of time series data (due to the time axis) creates a severe \emph{curse of dimensionality} \cite{ullman2011pcps}. Furthermore, any approach simply treating trajectory data as high-dimensional tabular records (not exploiting the temporal structure) inherently requires each person to contribute their full trajectory. In the present work, we seek to leverage temporal structure and prior trajectory inference work to avoid this downside, directly reducing the data required from the user to a single snapshot in time, directly improving the privacy characteristics by never revealing the vast majority of the user's data and potentially greatly reducing the complexity of data collection. %, necessitating a large amount of data under small privacy budgets. 

\subsection{Contributions}
We introduce a novel algorithm to generate private continuous-time data based on noisy particle gradient descent, overcoming key shortcomings of prior methods. Our approach has two key advantages:

\begin{itemize}
    \item Unlike previous methods that require users to provide full trajectory data (e.g. user-level data instead of event-level data), our method only requires each user to provide \emph{one data point}, which naturally has significant potential for privacy preservation (See Section \ref{sec:problem_statement}).
\item As far as we are aware, we are the first to provide any theoretical guarantees for this problem. We give statistical guarantees on both temporal marginal and trajectory recovery, which both become strong in the large temporal limit (See Theorem \ref{thm:statistical_convergence}).
\end{itemize}

\subsection{Related Work}
\paragraph{Trajectory inference} Trajectory inference aims to estimate the temporal dynamics and trajectories of a population from collections of unpaired observations, allowing for continuous-time data generation by sampling from the learned distribution. This problem has numerous applications in single-cell genomic data analysis \cite{weinreb2018limits,schiebinger2019reprogramming,saelens2019}. A line of works \cite{lavenant2024traj,chizat2022trajectoryinferencemeanfieldlangevin,ventre2024trajectoryinferencebranchingsde,shen2024multi,gu2025partially} consider trajectory inference with provable guarantees via the theory of entropic optimal transport \cite{peyre2019computational}. In particular, \cite{lavenant2024traj} initiated a mathematical framework for trajectory inference as optimization of the Kullback-Leibler (KL) divergence against the Wiener path measure; \cite{chizat2022trajectoryinferencemeanfieldlangevin} considers a computationally effective algorithm via mean-field Langevin dynamics; \cite{ventre2024trajectoryinferencebranchingsde} provided theoretical guarantees for optimization against branching Brownian motions; and \cite{gu2025partially} consider theoretical guarantees for optimization against more general path measures. Other works have also investigate the applications of neural networks for trajectory inference \cite{tong2023simulation,zhang2025learning}.

\paragraph{Private synthetic data}
Many existing works on private synthetic data generation work on tabular data \cite{bellovin2019privacy}, and, they often only come with utility guarantees with respect to a fixed set of linear queries \cite{barak2007privacy,thaler2012faster}. Further, it is known that under standard cryptographic assumptions (the existence of one-way functions), it is NP-hard to match two-way marginals of Boolean data \cite{ullman2011pcps}. On the machine learning side of things, a number of works have considered using neural networks for synthetic data. While these naturally apply to both tabular and non-tabular data, these methods lack theoretical guarantees. Initial approaches utilize generative adversarial networks (GANs) \cite{jordon2018pate,xie2018differentially,torkzadehmahani2019dp}, but GANs are known to exhibit modal collapse \cite{salimans2016improved,arjovsky2017towards} and are generally regarded as inferior for private synthetic data generation. More recently, training models such as diffusion models under DP have been promising, for instance see \cite{dockhorn2023differentially} for private training of diffusion models for image data.

For private trajectory synthesis, see \cite{jiang2021location} for a survey. Most methods that consider private trajectory synthesis require gridding the support of the data space \cite{yang2022collecting,du2023ldptrace,wang2023privtrace}.
These methods generally use private counts on disjoint grids, then fit a Markovian model to synthesize trajectories. Thus, these methods both suffer from the curse of dimensionality and lack statistical and utility guarantees. As far as we are aware, \cite{ghazi2022private} is the first paper to consider utility guarantees for any private time-series data. In particular, they aim to aggregate user-generated trajectories (e.g. find a mean trajectory), which can be considered a special case of our problem. However, \emph{all} of these methods require full trajectory data from each user, which thus require much more examples for good utility (for a fixed privacy budget) compared to our method. 

A line of works have also considered private measures with respect to the optimal transport metric \cite{boedihardjo2024private,he2023algorithmically,he2024online,he2025differentially,greenewald2024privacy}. These works treat the private dataset as an empirical measure, and the goal is to generate private synthetic data (another empirical measure) that is close in Wasserstein distance to the private empirical measure. The most algorithmically similar work to ours is \cite{donhauser2024privacy}, which also leverages optimal transport and particle gradient descent to compute private datasets, albeit they only consider \emph{tabular} datasets and also lack utility guarantees. At a high level, they aim to match marginals of the dataset using the sliced Wasserstein distance \cite{bonneel2015sliced}, and they leverage gradient descent on particles to match these marginals.

% need to discuss works that study synthetic data generation 

% synthetic data generation for trajectory data

% some based on OT (we can cite our barycenter work)

% \subsection{Organization}
% In Section \ref{sec:prelim} we provide preliminaries. In Section \ref{sec:main}, we provide main results on our algorithm and its privacy analysis. In Section \ref{sec:improved_privacy}, we discuss how to improve the privacy of our algorithm via various techniques. In Section \ref{sec:experiments}, we provide experiments to test our method. Finally, in Section \ref{sec:conclusion}, we conclude with some future directions. 

\section{Preliminaries}\label{sec:prelim}
\paragraph{Notation} We use $[n] := \{1, \dots, n\}$. For probability measures $\mu,\,\nu$, the relative entropy (e.g. the KL divergence) is $H(\mu|\nu)=\int\log(d\mu/d\nu)d\mu$ if $\mu\ll \nu$ and $+\infty$ otherwise. We use the notation $\mathcal{P}(\cdot)$ to denote the probability distributions over a space. The path space is $\Omega = C([0,1]:\mathcal{X})$, the set of continuous $\mathcal{X}$-valued paths. If $\mathbf{R}\in\mathcal{P}(\Omega)$ is a probability measure on the space of paths, its marginal at time $t$ is denoted as $\mathbf{R}_t\in\mathcal{P}(\mathcal{X})$.

\subsection{Differential privacy}
Differential privacy (DP) \cite{dwork2006dp} is a mathematical framework for establishing guarantees on privacy loss of an algorithm, with nice properties such as degradation of privacy loss under composition and robustness to post-processing. We provide a brief introduction and refer to \cite{dwork2014algorithmic} for a thorough treatment. 

\begin{definition}[$(\epsilon,\delta)$-DP]
Algorithm $\mathcal{A}$ is said to satisfy $(\epsilon,\delta)$-differential privacy if for all adjacent datasets $\mathcal{D}, \mathcal{D}'$ (datasets differing in at most one element) and all $\mathcal{S} \subseteq \operatorname{range}\mathcal{A}$, it holds $\Pr[\mathcal{A}(\mathcal{D})\in \mathcal{S}]\le e^\epsilon\Pr[\mathcal{A}(\mathcal{D}')\in \mathcal{S}] + \delta$.

\end{definition}

% We also state some basic primitives of DP.

% \begin{definition}[$\ell_p$-sensitivity]
% We define the $\ell_p$-sensitivity of a function $f$ to be \[
% \Delta_p f:= \max_{\mathcal{D}, \mathcal{D}'}\|f(\mathcal{D}) - f(\mathcal{D}')\|_p,
% \]
% where $\mathcal{D}, \mathcal{D}'$ are adjacent datasets. 
% \end{definition}

% We provide some basic primitives from differential privacy.

% \begin{lemma}[Gaussian mechanism]\label{lem:gaussian_mech}
% Let $f$ be a function, $\epsilon,\delta \in (0, 1)$, and $\sigma^2 \ge \Delta_2f\frac{2\ln(1.25/\delta)}{\epsilon}$. The Gaussian mechanism $f(\mathcal{D})+\mathcal{N}(0, \sigma^2)$ is $(\epsilon,\delta)$-DP. %\mb{Same comment as above. To use the previous definition, $\mathcal{A}$ needs to be deterministic.}
% \end{lemma}

% \begin{lemma}[Basic (sequential) composition]\label{lem:seq_comp}
%     Let $\mathcal{A}_1,\dots \mathcal{A}_k$ be $(\epsilon_i,\delta_i)$-DP algorithms. Then their composition $(\mathcal{A}_1(\mathcal{D}),\dots, \mathcal{A}_k(\mathcal{D}))$ is $(\sum_i \epsilon_i, \sum_i \delta_i)$-DP.
% \end{lemma}

We will state our results in Gaussian DP (GDP) \cite{dong2022gaussian} as our method is based on composition of DP-SGD. In particular, GDP gives the tightest composition for DP-SGD. Informally, GDP can be characterized as a continuum of $(\epsilon,\delta)$-DP guarantees, and we have lossless conversion of GDP to $(\epsilon,\delta)$-DP: 
\begin{lemma}[{\cite[Cor. 2.13]{dong2022gaussian}}]
A mechanism is $\mu$-GDP if and only if it is $(\epsilon,\delta_\mu(\epsilon))$-DP for all $\epsilon > 0$ where \[
\delta_\mu(\epsilon) = \Phi\left(-\frac{\epsilon}{\mu} + \frac{\mu}{2}\right)-e^{\epsilon}\Phi\left(-\frac{\epsilon}{\mu}-\frac{\mu}{2}\right).
\]
\end{lemma}

A nice property of GDP (and also DP) is the post-processing inequality, which informally says that transforming private output does not incur additional privacy loss. Formally, we have the following:
\begin{lemma}[Post-processing of DP and GDP, {\cite[Prop. 2.1]{dwork2014algorithmic}; \cite[Lem. 2.9]{dong2022gaussian}}]\label{lem:post_processing}
    If a mechanism $\mathcal{A}$ is $(\epsilon, \delta)$-DP (resp. $\mu$-GDP), then for any (possibly randomized algorithm) $g$, $g\circ \mathcal{A}(\mathcal{D})$ is $(\epsilon, \delta)$-DP (resp. $\mu$-GDP). 
\end{lemma}

\subsection{Trajectory inference}
Consider an SDE \begin{equation}
dX_t = -\nabla \Psi(X_t,t)dt + \sqrt{\tau}dB_t,\label{eq:sde}    
\end{equation}
with initial distribution $X_0 \sim \mu_0$, where $\{B_t\}$ is a Brownian motion, $\tau$ is the \emph{known} diffusivity parameter, and $\Psi\in C^2([0,1]\times \mathcal{X})$ is an \emph{unknown} potential function. Let $\mathbf{P}\in\mathcal{P}(\Omega)$ be the law of the SDE with initial condition $\mathbf{P}_0$ where $\mathbf{P}_t \in \mathcal{P}(\mathcal{X})$ are the marginals of $\mathbf{P}$ at time $t \in [0,1]$. Suppose we have $T$ observation times $0 \le t_1 <\cdots < t_T \le 1$, and at each observation time $t_i$, we observe $N_i$ i.i.d. samples from the $X_{t_i}\sim \mathbf{P}_{t_i}$, forming the empirical distributions $\hat{\mu}_i = \frac{1}{N_i}\sum_{j=1}\delta_{X_{t_i}^j}$. To simplify the presentation, we assume $N_i = N$ for each $i\in [T]$ and denote $\Delta t_i := t_{i+1}-t_i$ in the sequel.

The goal of trajectory inference \cite{lavenant2024traj,chizat2022trajectoryinferencemeanfieldlangevin,gu2025partially,yao2025learning} is to recover $\mathbf{P}$ given $(\hat{\mu}_1,\dots,\hat\mu_T)$. 
\cite{lavenant2024traj} formulated the trajectory inference problem as an entropy minimization problem of the functional $\mathcal{F}:\mathcal{P}(\Omega)\to\mathbb{R}$ over path-space defined as 
\begin{equation}
\mathcal{F}(\mathbf{R}) := \mathrm{Fit}^{\lambda,\sigma}(\mathbf{R}_{t_1},\dots, \mathbf{R}_{t_T}) + \tau H(\mathbf{R}|\mathbf{W}^{\tau}).\label{eq:functional_path_space}    
\end{equation}
% Let the observed empirical distribution smoothed by the $h$-wide heat kernel $\Phi_h$ be $\hat{\mu}_i^{h} := \Phi_h\left(\hat{\mu}_i\right)\in\mathcal{P}(\mathcal{X})$ for $i\in [\Bar{T}]$.

We need to introduce a \emph{fit function} to measure the discrepancy between the time marginals of the estimated trajectory distribution $\mathbf{R}$ and the observed samples. To that end, consider the fit function $\mathrm{Fit}^{\lambda, \sigma} :\mathcal{P}(\mathcal{X})^{T+1} \to \mathbb{R}$:
    \begin{equation}
    \mathrm{Fit}^{\lambda, \sigma}(\mathbf{R}_{t_0}, \dots, \mathbf{R}_{t_T}) := \frac{\Delta t_i}{\lambda}\sum_{i=1}^T \mathrm{DF}^\sigma(\mathbf{R}_{t_i},\hat{\mu}_i),        
    \end{equation}
    with data-fitting term introduced by \citep{chizat2022trajectoryinferencemeanfieldlangevin} to be
    \begin{equation}\label{eq:df}
        \begin{split}
    \mathrm{DF}^\sigma(\mathbf{R}_{t_i},\hat{\mu}_i) :={}& \int- \log \left[\int \exp\left(-\frac{\|x - y\|^2}{2\sigma^2}\right)d\mathbf{R}_{t_i}(x)\right]d\hat{\mu}_i(y)\\
    ={}& H(\hat{\mu}_i| \mathbf{R}_{t_i} * \mathcal{N}_\sigma) + H(\hat{\mu}_i) + C,
        \end{split}
    \end{equation}
    where $\mathcal{N}_\sigma$ is the Gaussian kernel with variance $\sigma^2$, $C>0$ is a constant. This data-fitting term can be interpreted as the negative log-likelihood under the noisy observation model $X_{t_i}^j + \sigma Z_{i,j}$, where $X_{t_i}^j$ is the observation and $Z_{i,j}\overset{i.i.d.}{\sim} \mathcal{N}(0,I)$. Observe that $\mathrm{DF}^\sigma$ is jointly convex in $(\mathbf{R}_{t_i}, \hat{\mu}_i)$ and linear in $\hat{\mu}_i$. 

For statistical consistency of trajectory inference, \cite{gu2025partially} extended \cite{lavenant2024traj} and showed that the infinite-time marginal limit of the minimizer of \eqref{eq:functional_path_space} for the data-fit term \eqref{eq:df} converges to the ground truth.\footnote{Their result holds for a more general partially-observed SDE, where the second term on the right-hand side of \eqref{eq:functional_path_space} is entropy with respect to a known divergence-free reference measure. We will not utilize this.}
\begin{theorem}[Consistency, informal]\label{thm:consistency}
Let $\mathbf{R}^{*}\in\mathcal{P}(\Omega)$ be the unique minimizer of \eqref{eq:functional_path_space}. If $\{t_i\}_{t\in [T]}$ becomes dense in $[0, 1]$ as $T\to \infty$, then \[
\lim_{\sigma\to 0,\lambda \to 0}\left(\lim_{T \to \infty} \mathbf{R}^{T,\lambda,h} \right)= \mathbf{P},\]
almost surely.
\end{theorem}

\subsection{Mean-field Langevin dynamics and its discretization}\label{sec:mfld}
In this section, we discuss the application of mean-field Langevin dynamics (MFLD) for trajectory inference, following \cite{chizat2022meanfieldlangevindynamicsexponential,gu2025partially}. We defer some technical details to Appendix \ref{app:mfld}. For a more complete discussion on mean-field Langevin dynamics and its applications, see for instance the discussion in \cite{nitanda2022convex,chizat2022meanfieldlangevindynamicsexponential,gu2025mirrormeanfieldlangevindynamics}. 

Since \eqref{eq:functional_path_space} is an infinite-dimensional optimization problem (on path space), we first need to reduce the problem over particle space $\mathcal{P}(\mathcal{X})^{T}$, where the particles and their couplings serve as an approximation of the trajectory distribution. We start by introducing the entropic-regularized optimal transport plan between $\mu,\nu$, which is defined as 
\begin{equation}
    T_\tau(\mu,\nu) := \min_{\gamma\in\Pi(\mu,\nu)}\int c_\tau(x,y)d\gamma(x,y) + \tau H(\gamma|\mu\otimes \nu) = \min_{\gamma\in\Pi(\mu,\nu)}\tau H(\gamma|p_\tau\mu\otimes \nu),\label{eq:eot}
\end{equation}
where $\Pi(\mu,\nu)$ is the set of transport plans between $\mu$ and $\nu$, $p_t(x,y)$ is the transition probability density of the Brownian motion over the time interval $[0,t]$, and $c_\tau(x,y) := -\tau \log p_\tau(x,y)$. This optimization problem is exactly a Schr\"odinger bridge problem and thus can be solved using entropic optimal transport \cite{peyre2019computational}, e.g., see Appendix \ref{app:eot} for a further discussion.

% \anming{some discussion on connection to Schr\"odinger bridge problem; and use potentials $(\varphi,\psi)$ hereafter}

We now consider the reduction. Define the functional $G:\mathcal{P}(\mathcal{X})^{T}\to \mathbb{R}$ defined for $\boldsymbol\mu=(\boldsymbol{\mu}^{(1)},\dots, \boldsymbol{\mu}^{(T)})$ that represents a family of $T$ constructed temporal marginals by 
\begin{equation}
    G(\boldsymbol{\mu}) := \operatorname{Fit}(\boldsymbol{\mu}) + \sum_{i=`}^{T-1}T_{\tau_i}(\boldsymbol{\mu}^{(i)},\boldsymbol{\mu}^{(i+1}),\label{eq:functionG}
\end{equation}
where $\tau_i := \tau \Delta t_i$. The reduced objective $F: \mathcal{P}(\mathcal{X})^{T}\to\mathbb{R}$ is defined as \begin{equation}
    F(\boldsymbol{\mu}) := G(\boldsymbol{\mu}) + \tau H(\boldsymbol{\mu}),\label{eq:functionF}
\end{equation}
where $H(\boldsymbol{\mu}) = \sum_{i=1}^T \log \boldsymbol{\mu}^{(i)}d\boldsymbol{\mu}^{(i)}$ is the negative differential entropy of the family of measures $\boldsymbol{\mu}$. \cite{chizat2022trajectoryinferencemeanfieldlangevin,gu2025partially} showed the equivalence of minimization of the \eqref{eq:functional_path_space} and \eqref{eq:functionF}. For a full discussion of this equivalence, please refer to Appendix \ref{app:rep_thm}.

The characterization of $F$ described in Appendix \ref{app:properties_of_functional} makes it amenable to utilize the MFLD, as done in \cite{chizat2022trajectoryinferencemeanfieldlangevin,gu2025partially}. In particular, using Proposition \ref{prop:propertiesG_F}, the MFLD is defined as the solution of the following non-linear SDE: \begin{equation}\label{eq:mckean_vlasov}
\begin{cases} 
dX_t^{(i)} = -\nabla V^{(i)}[\boldsymbol{\mu}_t](X_t^{(i)})dt + \sqrt{\tau}dB_t^{(i)}, \quad\quad&\operatorname{Law}(X_0^{(i)}) = \boldsymbol{\mu}_0^{(i)},\\
        \boldsymbol{\mu}^{(i)}_t = \operatorname{Law}(X_t^{(i)}), &i \in [T]
\end{cases}
\end{equation}
It can be straightforwardly checked that the \eqref{eq:mckean_vlasov} is characterized by the following system of Fokker-Planck PDEs: 
\begin{equation}\label{eq:fp_pde}
    \frac{\partial}{\partial t}\boldsymbol{\mu}_t^{(i)} = \nabla \cdot (\nabla V^{(i)}[\boldsymbol{\mu}_t]\boldsymbol{\mu}_t^{(i)}) + \tau \Delta \boldsymbol{\mu}_t^{(i)}
\end{equation}
for $i \in [ T]$. For the exponential convergence of the system, see Appendix \ref{app:mfld}.

The discussion of the MFLD in above is not implementable as an algorithm because it is in the continuous-time and infinite-particle limit. We need to time- and space-discretize the dynamics. To that end, let $m\in\mathbb{N}$ be the number of particles for each time marginal. Define $G_m :(\mathcal{X}^m)^{T}\to \mathbb{R}$ to be the discretized versions of \eqref{eq:functionG}, where we use $G_m(\hat{X}) := G(\hat{\boldsymbol{\mu}}_{\hat{X}})$ and \[
\hat{\boldsymbol{\mu}}_{\hat{X}}^{(i)} := \frac{1}{m} \sum_{j=1}^m \delta_{\hat{X}_j^{(i)}}.
\]
Observing that $m\nabla_{\hat{X}^{(i)}}G_m(\hat{X}) = \nabla V^{(i)}[\hat{\boldsymbol{\mu}}_{\hat{X}}](\hat{X}^{(i)})$ via \cite{chizat2022meanfieldlangevindynamicsexponential}, we obtain the discretization of \eqref{eq:mckean_vlasov}: \begin{equation}\label{eq:discretization}
\begin{cases}
    \hat{X}^{(i)}_{k+1} = \hat{X}^{(i)}_k - \eta \nabla V^{(i)}[\hat{\boldsymbol{\mu}}_k](\hat{X}_k^{(i)}) + \sqrt{\eta \tau} Z_k^{(i)}, \quad\quad& \hat{X}_0^{(i)}\overset{i.i.d.}{\sim} \left(\boldsymbol{\mu}_0^{(i)}\right)^{\otimes m}\\
    \hat{\boldsymbol{\mu}}_k^{(i)} = \hat{\boldsymbol{\mu}}_{\hat{X}_k^{(i)}}, & t\in [ T],
\end{cases}
    \end{equation}
where $\eta > 0$ is the step size and $Z_k^{(i)}\sim \mathcal{N}(0, I_d)^{\otimes m}$ are Gaussian vectors. The MFLD \eqref{eq:mckean_vlasov} is recovered in the limit as $m\to \infty$ and $\eta \to 0$ \cite{chizat2022trajectoryinferencemeanfieldlangevin,nitanda2022convex}. 

We utilize stochastic gradients with Poisson sampling parameter $\rho$ from each of the $T$ distributions\footnote{E.g. for each time marginal, we independently include a data point with probability $\rho$.} to compute a stochastic gradient $\nabla \hat{V}^{(i)}[\hat{\boldsymbol{\mu}}]$ instead of the full-batch gradient. We should think of $\rho = O\left(\frac{1}{N}\right)$, as this allows for privacy amplification.\footnote{We remark that using stochastic gradients also reduces the per-iteration runtime.} For some remarks on the discretization, see Appendix \ref{app:mfld}

\section{Synthetic Trajectory Generation}\label{sec:main}
\subsection{Problem statement}\label{sec:problem_statement}
We leverage the entropic minimization framework and MFLD for private continuous-time synthetic data generation with provable guarantees (Theorems \ref{thm:consistency} and \ref{thm:statistical_convergence}). We start by using the following notion of neighboring datasets.

\begin{definition}[Neighboring datasets]
    Let the dataset be $\mathcal{D}_{N,T} := \{x\mid x\in \cup_{i\in [\Bar{T}]}\operatorname{supp}\hat\mu_i\} \times [T]$ be a collection of data points at each time. We say two datasets $\mathcal{D},\mathcal{D}'$ are neighboring if they differ by exactly one row (in the first marginal).
\end{definition}
When it is clear, we drop the dependence on $N$ and $T$. Recall that we assume that each time point contains the same number of data points, $N$, so the full dataset contains $NT$ data points. In our definition, each person contributes \emph{one} datapoint to \emph{one} temporal marginal. 

Our method straightforwardly extends to the setting where each person contributes a full collection of timepoints via group privacy, along with the setting where each time point may have a different number of supports, but we suppress these details to simplify the discussion. However, we do require that the time marginals of neighboring datasets to match as our privacy analysis utilizes amplification by subsampling and the algorithm requires support on all of the temporal marginals.

\paragraph{Model} Suppose we have a distribution of continuous-time trajectories $\mathbf{P}$. We have a dataset $\mathcal{D}$ that comes from temporal marginals of $\mathbf{P}$ on a fixed time set $\{t_1, \dots, t_T\}$. Our goal is to recover $\mathbf{P}$ from $\mathcal{D}$ under the constraints of DP.

Via the discussion from the previous section, it is natural to adapt the mean-field Langevin dynamics for trajectory inference to this problem. We provide pseudocode for the implementation of \eqref{eq:discretization} (with addition of clipping) in Algorithm \ref{alg:traj_algo}. 

We remark that the fixed time set $\{t_1,\dots, t_T\}$ during training is not restrictive. In practice, we can bin users to the closest time marginal. For instance, in a healthcare example where we want to interpolate information from yearly checkups (e.g. height and weights for populations of male and female children), we can bin the measurements of all $i$th year checkups together even if they don't occur exactly on the $i$th birthday.

Note further that this procedure learns a series of entropic OT plans (one for each EOT cost), which are probabilistic in nature and define a continuous-time process. As a result, sampled synthetic trajectories can be sampled at any collection of time points, these need not match the training grid $\{t_1, \dots, t_T\}$ and can in fact be \emph{continuous time}. See \cite{chizat2022trajectoryinferencemeanfieldlangevin} for further details.

\begin{algorithm}[t]
\begin{algorithmic}[1]
\footnotesize
\Require{Collection of observations $\hat{\boldsymbol{\mu}} := (\hat{\mu}_1,\dots, \hat{\mu}_T)$, collection of $T$ time samples $(t_1,\dots, t_T)$, step size $\eta$, number of particles $m$, number of iterations $K$, Poisson subsample parameter $\rho$, clipping threshold $C$, noise parameter $\tau$}\State{Initialize $m$ particles for each time: $\mathbf{m}_0 := (\hat{m}_1,\dots, \hat{m}_T)\in\mathcal{X}^{m\times T}$}
\For{$K$ iterations}
\For{$i \in [T-1]$}
\State{$\Delta t_i := t_{i+1}-t_i$}
\State{$C_i:=\{C_{j,k}\}_{j,k=1}^m \gets \frac{1}{2}\|\hat{m}_{{i+1},k} - \hat{m}_{i,j}\|^2$}
\State{$\gamma_i \gets \mathrm{Sinkhorn}(\hat{m}_i,\hat{m}_{i+1}, C_i, \tau\cdot \Delta t_i)$} %\Comment{$T_i\in \Pi(\hat{m}_i, \hat{m}_{i+1})$}
\EndFor
\State{Poisson subsample $\Tilde{\boldsymbol{\mu}}:=(\Tilde{\mu}_1,\dots,\Tilde{\mu}_T)$ i.i.d. from $\hat{\boldsymbol{\mu}}$ with parameter $\rho$}
\State{Compute $\{\nabla \hat{V}^{(i)}_{\Tilde{\mu}}\}_{x\in \Tilde{\mu}_i,i\in[T],}$ using $\Tilde{\boldsymbol{\mu}}$ and $\boldsymbol{\gamma}$ via Proposition \ref{prop:propertiesG_F} and clip (item-wise) with threshold $C$}
\State{Update particles for $k \in [m]$, where $\xi_{i,k}\overset{i.i.d.}{\sim} \mathcal{N}(0, C\tau\cdot I)$: \[
\hat{m}_{i+1,k} \gets \hat{m}_{i,k} - \eta \frac{1}{|\Tilde{\mu}_i|}\left(\sum_{x\in \Tilde{\mu}_i}\nabla \hat{V}^{(i)}_x[\hat{\mathbf{m}}_i](m_{i,k}) + \xi_{i,k}\right)
\]
}
\EndFor
\State{\Return particles $\hat{\mathbf{m}}_K$, trajectories $\gamma_{T-1}\circ \cdots \circ \gamma_1$}
\end{algorithmic}
\caption{Continuous-time synthetic trajectory generation}
\label{alg:traj_algo}
\end{algorithm}

\subsection{Privacy guarantees}\label{sec:privacy_guarantees}
Note that the diffusivity $\tau$ in \eqref{eq:discretization} adds Gaussian noise to the gradients, and hence will yield privacy guarantees as \eqref{eq:discretization} is exactly DP-SGD. To bound the sensitivity of gradients, we utilize clipping \cite{Abadi_2016}, where we clip the norm of each gradient by $C$ to bound its sensitivity.\footnote{Observe that based on Proposition \ref{prop:propertiesG_F}, we only need to clip the component based on $\frac{\delta {\operatorname{Fit}}}{\delta\boldsymbol{\mu}^{(i)}}$ as the Schr\"odinger potentials are private (from the previous iteration) by post-processing, Lemma \ref{lem:post_processing}.} Using guarantees of GDP for DP-SGD, parallel composition (Lemma \ref{lem:parallel_composition}) over the time intervals, and post-processing, we have the following privacy guarantee.
\begin{lemma}[Privacy, \cite{bu2020deep}]\label{lem:privacy} Algorithm \ref{alg:traj_algo} with $K$ iterations with clipping threshold $C$, Poisson subsampling with parameter $\rho$, and $\tau >0$ is asymptotically $\mu$-GDP with $\mu = \rho \sqrt{K (e^{1/\tau^2}-1)}$ if $K\gg1$.
\end{lemma}

\subsection{Utility guarantees}
The recent work \cite{yao2025learning} has provided a statistical rate of convergence for the estimator \eqref{eq:functionF}, and they showed that the estimator recovers the temporal marginals for general stochastic processes that do not necessarily follow SDEs of the form \eqref{eq:sde}. As a result, their result also apply in our setting.

\begin{theorem}[Statistical rate of convergence, informal]\label{thm:statistical_convergence}
Under the setting of Theorems \ref{thm:consistency} and \ref{thm:representer}, with high probability, it holds \[
    \int_0^1 d_H^2(\mathbf{R}_t^* * g_\sigma, \mathbf{P}_t * g_\sigma)dt \lesssim \max\left\{\frac{1}{T}, \frac{1}{N^{2/3}T^{1/3}}\right\}.
    \]
\end{theorem}
Here $\lesssim$ hides poly-logarithmic factors. See \cite[Thm. 3]{yao2025learning} and the discussion therein for the formal statement. To interpret this result, $d_H^2(p,q) = \int(\sqrt{p}-\sqrt{q})^2dx$  is the squared Hellinger distance between two probability distributions. Recall that $\mathbf{R}_t^*$ is induced minimizer given the marginals and entropic transport plans, and $\mathbf{P}_t^*$ is the true distribution. This result says that the full estimated path-space distribution $t\mapsto \mathbf{R}_t^*$ over $[0, 1]$ converges to the ground-truth distribution  uniformly in time. We briefly remark that the statistical guarantees from \cite{yao2025learning} only hold for the infinite-particle limit and not the finite-particle discretization.

\section{Improvements Upon the Base Method}\label{sec:improved_privacy}
Langevin dynamics-based optimization methods (and thus mean-field Langevin-based) are heavily dependent on good initializations, e.g. initial distributions where $F_\tau(\boldsymbol{\mu}_0)-F_\tau(\boldsymbol{\mu}^*)\le \operatorname{KL}(\boldsymbol{\mu}_0|\boldsymbol{\mu}^*)$ is small, where $\boldsymbol{\mu}^*$ is the minimizer of \eqref{eq:functionF}. Here, the inequality is due to the entropy sandwich inequality \cite{nitanda2022convex,chizat2022meanfieldlangevindynamicsexponential}. For instance, see \cite[Fig. 1]{gu2025partially} for a trajectory inference example where poor initalization causes MFLD fail to optimize. We provide a straightforward method that improves convergence guarantees via warm starts, at the cost of a small increase in privacy budget. Alternatively, we can slightly reduce the privacy budget for the optimization phase by the corresponding amount to maintain the same overall privacy budget.

\subsection{Warm starts via private initializations}

A very straightforward (and computationally cheap) approach is to use private mean or median \cite{haghifam2025private} of the data at each time distribution as a warm start, which is especially useful if the data is unimodal. To that end, we have the following.

\begin{proposition}
Let $\epsilon,\delta \in (0, 1)$. Suppose $\cup_{t\in [T]}\hat{\mu}_i \subseteq B_0(R)$. Let $\Bar{\mu}_i = \frac{1}{N}\sum_{x\in \hat{\mu}_i}x$. Then using initialization of $(\{\delta_{\Bar{\mu}_1 + Y_1}\}_{j=1}^m,\dots, \{\delta_{\Bar{\mu}_T+Y_T}\}_{j=1}^m)$ in line 1 of Algorithm \ref{alg:traj_algo}, where $Y_i \overset{i.i.d.}{\sim}\mathcal{N}\left(0, \frac{8R^2\log(1/\delta)}{\epsilon^2N^2}\right)$ for $i\in[T]$ is $(\epsilon,\delta)$-DP.
\end{proposition}
\begin{proof}
    This follows from a straightforward application of the Gaussian mechanism (Lemma \ref{lem:gaussian_mechanism}), parallel composition (Lemma \ref{lem:parallel_composition}) due to the fact that data at each time marginal are disjoint, and the fact that the sensitivity of $\Bar{\mu}_i$ is $\frac{2R}{N}$.
\end{proof}

If the data is not unimodal, we can utilize private clustering to identify $k$ centers. Then we can use the Gaussian mechanism to privately obtain the counts of the number of points in each cluster, and use a Gaussian mixture with approximate weights as the initialization.

% We remark that another approach that can yield 

% Thus, we can achieve more privacy using fewer total iterations by obtaining a warm start as follows: consider a larger $\tau'$, and run the algorithm for fewer iterations $k'$, both of which contribute to consuming less privacy budget. After convergence is reached, we let the output distribution be our warm start. Then, we can consume more privacy budget by taking $\tau\le \tau'$ (annealing), and running the algorithm for $k\ll k'$ iterations.\footnote{We remark that \cite{suzuki2024feature} considers and analyzes this kind of annealing procedure in the discrete-time setting for mean-field neural networks. Further, other papers consider annealing for non-log-concave sampling.}

% We can also improve the runtime during the warm start as follows. During the first-round, we use a small number of particles $m'$, running the algorithm to convergence in $k'$ iterations and yielding distributions $\mu_0',\dots, \mu_T'$. Then we sample distributions $\mu_0,\dots, \mu_T$ supported on $m\gg m'$ particles uniformly at random from $\mu_0',\dots,\mu_T'$. Now we run this for $k\ll k'$ iterations with smaller $\xi$ (and larger privacy budget). For optimal runtime, we would want $(m')^2k' \approx m^2k$.

\subsection{Subsampled $T$}
When $T$ is large, the method becomes very computationally intensive. One way to alleviate this issue is to subsample the time intervals. In particular, we can take a random subset $\Tilde{T} \subseteq \{2, \dots, T-1\}$ and take the data $( \hat{\mu}_{\Tilde{t}_1}, \dots, \hat{\mu}_{\Tilde{t}_m})$ adjoined with post-processed privatized versions of $\hat{\mu}_1$ and $\hat{\mu}_T$. For instance, to privatize the endpoints, we can use the private initialization as above (as input to the algorithm). Privatizing these endpoints ensures the trajectories are supported over the whole intervals. We can bound the utility degradation from this subsampling via the following result, whose proof we provide in Appendix \ref{app:proofs}.

\begin{proposition}\label{prop:rand_subset}
    Let $\Tilde{T}\subset [T-1]$ be a random subset of size $z = \Omega(\log^2T)$. Then $\mathbb{E}[\max(\Tilde{T}_{i} - \Tilde{T}_{i-1})]
     = O\left(\frac{T}{z}\log z\right)$.
\end{proposition}

Using this result, we immediately have the following corollary of Theorem \ref{thm:statistical_convergence}.
\begin{corollary}
    Let $|\Tilde{T}| = T^\alpha$ for $\alpha \in (0,1)$. Then with constant probability, it holds \[
    \int_0^1 d_H^2(\mathbf{R}_t^* * g_\sigma, \mathbf{P}_t * g_\sigma)dt \lesssim \max\left
    \{\frac{1}{T^{1-\alpha}}, \frac{1}{N^{2/3}T^{\frac{1-\alpha}{3}}}{}\right\}.\]
\end{corollary}

Note that this subsampling procedure will also amplify privacy \cite{balle2018privacyamplificationsubsamplingtight}.

\section{Experiments}\label{sec:experiments}

\begin{figure}
    \centering
    \includegraphics[width=0.49\textwidth]{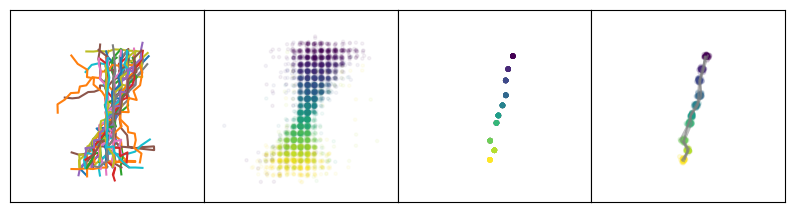}  \includegraphics[width=0.49\textwidth]{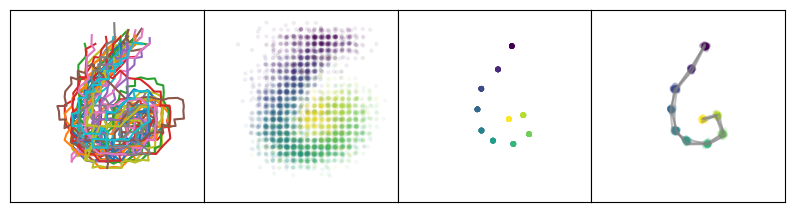}
    \includegraphics[width=0.49\textwidth]{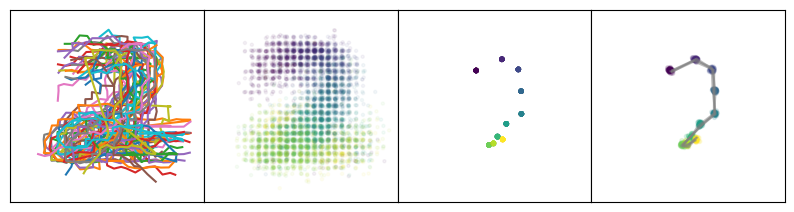}  \includegraphics[width=0.49\textwidth]{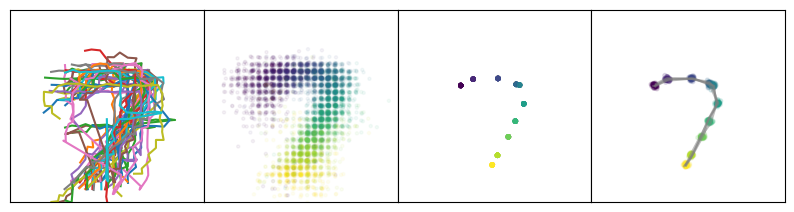}
    \includegraphics[width=0.49\textwidth]{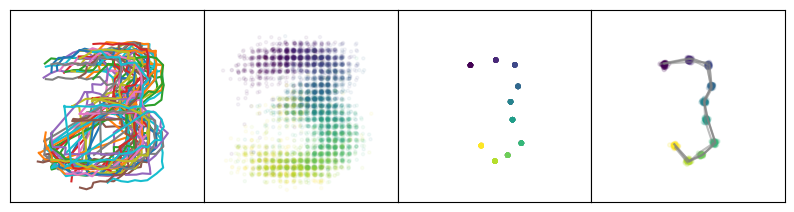}
    \includegraphics[width=0.49\textwidth]{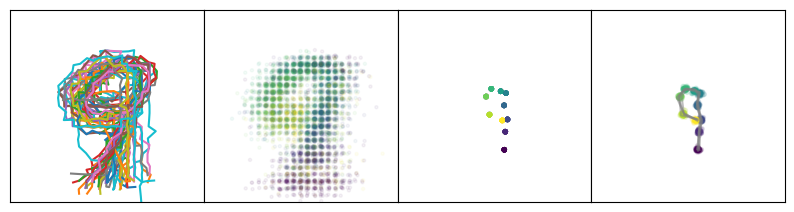}
    \caption{MNIST data using warm starts. Column 1 contains 50 example trajectories; column 2 contains temporal marginals; column 3 contains private means at each temporal marginal; and the last column contains recovered sample paths based on entropic OT plans.}
    \label{fig:warmstart_all}
\end{figure}

% \begin{figure}[t]
%     \centering
%     \centering
%     \begin{subfigure}[t]{0.4\textwidth}
%     \centering
%     \includegraphics[width=0.45\textwidth]{figures/main_true.png}
%     \caption{50 ground truth trajectories.}
%     \label{fig:7_traj}
%     \end{subfigure}~
%     \begin{subfigure}[t]{0.4\textwidth}
%     \centering
%     \includegraphics[width=0.45\textwidth]{figures/main_out.png}
%     \caption{Temporal marginals.}
%     \label{fig:marginals}
%     \end{subfigure}\\
%     \begin{subfigure}[t]{0.9\textwidth}
%     \centering
%     \includegraphics[width=\linewidth]{figures/main.png}
%     \caption{Recovered temporal marginals and trajectories (in grey).}
%     \label{fig:reconstruction}
%     \end{subfigure}
%     \caption{Example reconstruction of MNIST dataset 7 with $T = 10$ and $N = 626$ under $(\epsilon=1,\delta=10^{-3})$-DP.}
%     \label{fig:fig7}
% \end{figure}

We test our experiments on a modification of the MNIST dataset \cite{lecun-mnisthandwrittendigit-2010}. We utilize the MNIST strokes of \url{https://github.com/edwin-de-jong/mnist-digits-as-stroke-sequences} as ``trajectories'' for handwritten data. We provide experiments on digits 1, 2, 3, 6, 7, 9: see Figure \ref{fig:main_all}. We skip digits 4 and 5 because they require ``lifting the pen,'' and thus do not yield continuous paths. We skip digit 8 because a large fraction of the paths in the dataset are corrupted and come out as mirrored half-3's, e.g. see Figure \ref{fig:incorrect8} in Appendix \ref{app:experiments}. In Section \ref{sec:privacy_guarantees}, we state (asymptotic) privacy guarantees in terms of GDP, but we use AutoDP \cite{autodp1,autodp2,autodp3} to convert GDP guarantees to exact $(\epsilon,\delta)$-DP guarantees. We preprocess the data and rescale to the $[0, 1]^2$ grid, yielding trajectories as in the first column of Figure \ref{fig:warmstart_all}. For each handwritten data, we take only one datapoint after binning (see second column of Figure \ref{fig:warmstart_all} . For both of these, we plot the results with a small Gaussian perturbation with variance $0.005$ as MNIST is supported on $[28]^2$.

% We run Algorithm \ref{alg:traj_algo} with $m = 50$ particles, $\eta=0.01$, $K = 250$ iterations, clipping parameter $C = 1$, Poisson subsampling parameter $\rho = \frac{5}{N}$, and $\tau = 1$. Our basic initialization is $\mathcal{N}((0.5, 0.5), 0.2\cdot I_2)$. We plot recovered temporal marginals and sample paths based on entropic OT plans as in Figure \ref{fig:main_all} (right). Observe that the results for all the digits except 2 and 9 are reasonable.

All of our experiments were run on CPU and take no more than a few minutes. We run Algorithm \ref{alg:traj_algo} with warm start initializations in Figure \ref{fig:warmstart_all}, using $\eta=0.001$, $m = 50$, $T = 10$, $C = 1$,  $K = 20$, and $\rho = \frac{20}{N}$. We use $(1,5\cdot 10^{-4})$-DP for both private initialization and optimization, yielding $(2, 10^{-3})$-DP guarantees via Lemma \ref{lem:basic_composition}. We plot recovered temporal marginals and sample paths based on entropic OT plans in the fourth column.

% \begin{figure}
%     \centering
%     \includegraphics[width=0.35\linewidth]{figures/warmstart9.png}
%     \caption{Warm start for digit 9. Left contains the private initialization, and right contains the reconstructed data.}
%     \label{fig:warmstart_main}
% \end{figure}

We provide an experiment of subsampling in Figure \ref{fig:subsampling} (without warm starts). Here, we use $K = 100$. Instead of privatizing the endpoints as described in Section \ref{sec:improved_privacy}, we do a full direct subsampling (of $5$ timepoints out of $10$). This computation uses a budget of $(0.34, 5 \cdot 10^{-4})$-DP due to privacy amplification by subsampling \cite{balle2018privacyamplificationsubsamplingtight}. We provide additional experiments in Appendix \ref{app:experiments} without warm starts. 

% \anming{add an easier multimodal experiment}

\begin{table}
    \centering
    \begin{tabular}{|c||c c c c c c|}
    \hline
        Digit & 1 & 2 & 3 & 6 & 7 & 9 \\
        \hline
        Distance & 0.034 & 0.02 & 0.024 & 0.049 & 0.024 & 0.023 \\
        \hline
    \end{tabular}
    \vspace{0.1cm}
    \caption{Average $W_2$ distance (over time) between marginals of true distributions and recovered distributions from Figure \ref{fig:warmstart_all}.}
    \label{tab:w2_distance}
    % \vspace{-0.5cm}
\end{table}

\begin{figure}
    \centering
    \includegraphics[width=0.9\linewidth]{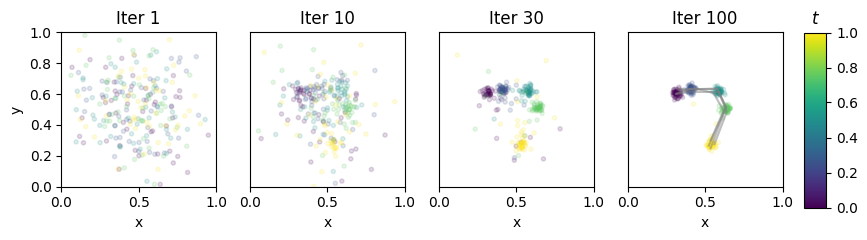}
    \caption{Subsampled 7 with 5 points (out of 10 total corresponding to those of Figure \ref{fig:warmstart_all}), along with plots of particles at various stages of optimization. The total privacy budget used is $(0.34, 5\cdot 10^{-4})$-DP.}
    \label{fig:subsampling}
\end{figure}

We provide a multimodal experiment in Figure \ref{fig:multimodal_main}.
The data is constructed using $ N = 3000$ points for each of $T= 10$ timesteps. One-third each of the data consists of an upper-half semicircle, a straight line, and a lower-half semicircle, and all the data are convolved with a Gaussian $\mathcal{N}(0, 0.005\cdot I_2)$. In the output, we utilize exact optimal transport plans between the marginals to obtain smoother trajectories.\footnote{We find in practice using \emph{entropic} optimal transport plans yields more undesirable ``cross cluster'' trajectories in these multimodal experiments.}

In the annealed experiment, we use uniform random points in $[-1, 1]^2$ as initialization. We run 1000 iterations with $\eta = 0.005$, $\tau = 2.5$, followed by 500 iterations with $\eta = 0.003$ with $\tau = 1.5$, and 500 iterations with $\eta = 0.001$ and $\tau = 1$, using a total of $(\epsilon = 1.84, \delta = 0.01)$-DP. Here, each regime utilizes the same $\delta' = 0.01/3$ (yielding $\delta =0.01$ by Lemma \ref{lem:basic_composition}). 

For the warm start via private initialization, we set a total privacy budget of $\epsilon = 1$ and $\delta = 0.01$. We utilize the implementation \href{https://github.com/UzL-PrivSec/dp-kmeans-max-cove}{https://github.com/UzL-PrivSec/dp-kmeans-max-cover} of \cite{nguyen2021differentially} to compute private $3$-means for each time period under $(\epsilon/3, \delta/3)$-DP, which is in total $(\epsilon/3,\delta/3)$-DP by Lemma \ref{lem:parallel_composition}. Next, to obtain a proxy counts of the data, we use the Gaussian mechanism (Lemma \ref{lem:gaussian_mechanism}) to approximate the counts of number of datapoints matched to each cluster. Each of the counts uses a budget of $\epsilon/9$ and $\delta/9$ so composing over the 3 clusters is $(\epsilon/3,\delta/3)$-DP (uniformly in time). We use these means and approximate counts for initialization as warm starts, which doesn't consume additional privacy budget due to post-processing (Lemma \ref{lem:post_processing}). Finally, we run the algorithm using 50 iterations with $\eta = 0.1$ and $\tau = 2$. 

We observe that using warm starts via private initialization yields more accurate weights between clusters, e.g. the annealing experiment results in optimization finding a local minimum.

% some additional notes: partial DP can be adapted to the partially observed setting, e.g. see https://arxiv.org/pdf/2209.04053

% can add this to the appendix

\begin{figure}[t]
    \centering
    \begin{subfigure}[t]{0.3\textwidth}
    \centering
    \includegraphics[width=0.7\textwidth]{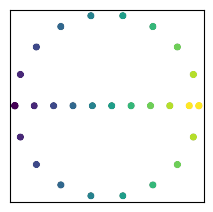}
    \caption{Ground truth.}
    \end{subfigure}\hspace{0.2cm}
    \begin{subfigure}[t]{0.3\textwidth}
    \centering
    \includegraphics[width=0.7\textwidth]{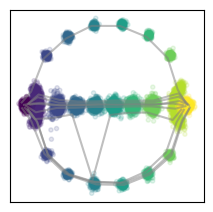}
    \caption{Annealing with $(1.84, 10^{-2})$-DP.}
    \end{subfigure}\hspace{0.2cm}
    \begin{subfigure}[t]{0.3\textwidth}
    \centering
    \includegraphics[width=0.7\textwidth]{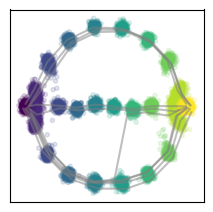}
    \caption{Warm starts with $(1, 10^{-2})$-DP.}
    \end{subfigure}
    \caption{Multimodal experiments. We plot 10 trajectories each.}
    \label{fig:multimodal_main}
\end{figure}

\section{Conclusion}\label{sec:conclusion}
In this work, we initiate the study of private generation of continuous-time synthetic data under the setting where each user contributes one data point and not their entire trajectory, leveraging the framework of \cite{lavenant2024traj,chizat2022trajectoryinferencemeanfieldlangevin,gu2025partially,yao2025learning}. Experimental validation showed the effectiveness of our method in both ``handwritten'' data and multimodal synthetic data. For future work, it would be interesting to adapt neural network approaches for trajectory inference to this problem. We also note that our method is in the static setting, where the data curator has access to all the data at once. It would be interesting to study private synthetic data generation under the online \cite{jain2012differentially} or continual release \cite{dwork2010continual_release,chan2011continual_release} settings, where data comes as part of a stream. We also leave for future work incorporating the sliced-Wasserstein method of \cite{donhauser2024privacy} to allow our method to work with discrete-valued (e.g. Boolean) data. A limitation of our work is that our method is not adaptive to data living on a low-dimensional subspace (e.g. the manifold hypothesis) as our method constructs high-dimensional Gaussians. Additionally, our method requires convexity of the space, which may not always hold, e.g. if we want to synthesize trajectories on road networks. We leave these for future work.

% remark that the data is not continual release setting; although this is still interesting (how to do this in the continual release setting)

\bibliography{references}

\newcommand{\etalchar}[1]{$^{#1}$}
\begin{thebibliography}{WWT{\etalchar{+}}18}

\bibitem[AB17]{arjovsky2017towards}
Martin Arjovsky and Leon Bottou.
\newblock Towards principled methods for training generative adversarial networks.
\newblock In {\em International Conference on Learning Representations}, 2017.

\bibitem[Abo18]{abowd2018us}
John~M Abowd.
\newblock The {US} census bureau adopts differential privacy.
\newblock In {\em Proceedings of the 24th ACM SIGKDD international conference on knowledge discovery \& data mining}, pages 2867--2867, 2018.

\bibitem[ACG{\etalchar{+}}16]{Abadi_2016}
Martin Abadi, Andy Chu, Ian Goodfellow, H.~Brendan McMahan, Ilya Mironov, Kunal Talwar, and Li~Zhang.
\newblock Deep learning with differential privacy.
\newblock In {\em Proceedings of the 2016 ACM SIGSAC Conference on Computer and Communications Security}, CCS’16. ACM, October 2016.

\bibitem[App17]{apple_dp}
Apple.
\newblock Learning with privacy at scale.
\newblock 2017.

\bibitem[BBG18]{balle2018privacyamplificationsubsamplingtight}
Borja Balle, Gilles Barthe, and Marco Gaboardi.
\newblock Privacy amplification by subsampling: Tight analyses via couplings and divergences, 2018.

\bibitem[BCD{\etalchar{+}}07]{barak2007privacy}
Boaz Barak, Kamalika Chaudhuri, Cynthia Dwork, Satyen Kale, Frank McSherry, and Kunal Talwar.
\newblock Privacy, accuracy, and consistency too: a holistic solution to contingency table release.
\newblock In {\em Proceedings of the twenty-sixth ACM SIGMOD-SIGACT-SIGART symposium on Principles of database systems}, pages 273--282, 2007.

\bibitem[BDLS20]{bu2020deep}
Zhiqi Bu, Jinshuo Dong, Qi~Long, and Weijie~J Su.
\newblock Deep learning with gaussian differential privacy.
\newblock {\em Harvard data science review}, 2020(23):10--1162, 2020.

\bibitem[BDR19]{bellovin2019privacy}
Steven~M Bellovin, Preetam~K Dutta, and Nathan Reitinger.
\newblock Privacy and synthetic datasets.
\newblock {\em Stan. Tech. L. Rev.}, 22:1, 2019.

\bibitem[BRPP15]{bonneel2015sliced}
Nicolas Bonneel, Julien Rabin, Gabriel Peyr{\'e}, and Hanspeter Pfister.
\newblock Sliced and radon wasserstein barycenters of measures.
\newblock volume~51, pages 22--45. Springer, 2015.

\bibitem[BSV24]{boedihardjo2024private}
March Boedihardjo, Thomas Strohmer, and Roman Vershynin.
\newblock Private measures, random walks, and synthetic data.
\newblock {\em Probability theory and related fields}, 189(1):569--611, 2024.

\bibitem[Chi22]{chizat2022meanfieldlangevindynamicsexponential}
Lenaic Chizat.
\newblock Mean-field langevin dynamics: Exponential convergence and annealing, 2022.

\bibitem[CNZ24]{chewi2024uniform}
Sinho Chewi, Atsushi Nitanda, and Matthew~S Zhang.
\newblock Uniform-in-$ n $ log-sobolev inequality for the mean-field langevin dynamics with convex energy.
\newblock {\em arXiv preprint arXiv:2409.10440}, 2024.

\bibitem[CSS11]{chan2011continual_release}
T-H~Hubert Chan, Elaine Shi, and Dawn Song.
\newblock Private and continual release of statistics.
\newblock {\em ACM Transactions on Information and System Security (TISSEC)}, 14(3):1--24, 2011.

\bibitem[CZHS22]{chizat2022trajectoryinferencemeanfieldlangevin}
Lenaic Chizat, Stephen Zhang, Matthieu Heitz, and Geoffrey Schiebinger.
\newblock Trajectory inference via mean-field langevin in path space, 2022.

\bibitem[DAHY24]{donhauser2024privacy}
Konstantin Donhauser, Javier Abad, Neha Hulkund, and Fanny Yang.
\newblock Privacy-preserving data release leveraging optimal transport and particle gradient descent.
\newblock {\em arXiv preprint arXiv:2401.17823}, 2024.

\bibitem[DCVK23]{dockhorn2023differentially}
Tim Dockhorn, Tianshi Cao, Arash Vahdat, and Karsten Kreis.
\newblock Differentially private diffusion models.
\newblock {\em Transactions on Machine Learning Research}, 2023.

\bibitem[DHZ{\etalchar{+}}23]{du2023ldptrace}
Yuntao Du, Yujia Hu, Zhikun Zhang, Ziquan Fang, Lu~Chen, Baihua Zheng, and Yunjun Gao.
\newblock Ldptrace: Locally differentially private trajectory synthesis.
\newblock {\em arXiv preprint arXiv:2302.06180}, 2023.

\bibitem[DKY17]{ding2017microsoft}
Bolin Ding, Janardhan Kulkarni, and Sergey Yekhanin.
\newblock Collecting telemetry data privately.
\newblock {\em Advances in Neural Information Processing Systems}, 30, 2017.

\bibitem[DMNS06]{dwork2006dp}
Cynthia Dwork, Frank McSherry, Kobbi Nissim, and Adam Smith.
\newblock Calibrating noise to sensitivity in private data analysis.
\newblock In Shai Halevi and Tal Rabin, editors, {\em Theory of Cryptography}, pages 265--284, Berlin, Heidelberg, 2006. Springer Berlin Heidelberg.

\bibitem[DNPR10]{dwork2010continual_release}
Cynthia Dwork, Moni Naor, Toniann Pitassi, and Guy~N Rothblum.
\newblock Differential privacy under continual observation.
\newblock In {\em Proceedings of the forty-second ACM symposium on Theory of computing}, pages 715--724, 2010.

\bibitem[DR{\etalchar{+}}14]{dwork2014algorithmic}
Cynthia Dwork, Aaron Roth, et~al.
\newblock The algorithmic foundations of differential privacy.
\newblock {\em Foundations and Trends{\textregistered} in Theoretical Computer Science}, 9(3--4):211--407, 2014.

\bibitem[DRS22]{dong2022gaussian}
Jinshuo Dong, Aaron Roth, and Weijie~J Su.
\newblock Gaussian differential privacy.
\newblock {\em Journal of the Royal Statistical Society: Series B (Statistical Methodology)}, 84(1):3--37, 2022.

\bibitem[GCG25]{gu2025partially}
Anming Gu, Edward Chien, and Kristjan Greenewald.
\newblock Partially observed trajectory inference using optimal transport and a dynamics prior.
\newblock In {\em The Thirteenth International Conference on Learning Representations}, 2025.

\bibitem[GK25]{gu2025mirrormeanfieldlangevindynamics}
Anming Gu and Juno Kim.
\newblock Mirror mean-field langevin dynamics.
\newblock {\em arXiv preprint arXiv:2505.02621}, 2025.

\bibitem[GKK{\etalchar{+}}22]{ghazi2022private}
Badih Ghazi, Neel Kamal, Ravi Kumar, Pasin Manurangsi, and Annika Zhang.
\newblock Private aggregation of trajectories.
\newblock {\em Proceedings on Privacy Enhancing Technologies}, 2022.

\bibitem[GYWX24]{greenewald2024privacy}
Kristjan Greenewald, Yuancheng Yu, Hao Wang, and Kai Xu.
\newblock Privacy without noisy gradients: Slicing mechanism for generative model training.
\newblock {\em arXiv preprint arXiv:2410.19941}, 2024.

\bibitem[HC24]{hod2024dp_israel}
Shlomi Hod and Ran Canetti.
\newblock Differentially private release of {I}srael's national registry of live births.
\newblock {\em arXiv preprint arXiv:2405.00267}, 2024.

\bibitem[HS87]{holleystroock}
Richard Holley and Daniel Stroock.
\newblock Logarithmic sobolev inequalities and stochastic ising models.
\newblock {\em Journal of Statistical Physics}, 46(5):1159--1194, 1987.

\bibitem[HSU25]{haghifam2025private}
Mahdi Haghifam, Thomas Steinke, and Jonathan Ullman.
\newblock Private geometric median.
\newblock In {\em Advances in Neural Information Processing Systems}, volume~37, pages 46254--46293, 2025.

\bibitem[HSVZ25]{he2025differentially}
Yiyun He, Thomas Strohmer, Roman Vershynin, and Yizhe Zhu.
\newblock Differentially private low-dimensional synthetic data from high-dimensional datasets.
\newblock {\em Information and Inference: A Journal of the IMA}, 14(1):iaae034, 2025.

\bibitem[HVZ23]{he2023algorithmically}
Yiyun He, Roman Vershynin, and Yizhe Zhu.
\newblock Algorithmically effective differentially private synthetic data.
\newblock In {\em The Thirty Sixth Annual Conference on Learning Theory}, pages 3941--3968. PMLR, 2023.

\bibitem[HVZ24]{he2024online}
Yiyun He, Roman Vershynin, and Yizhe Zhu.
\newblock Online differentially private synthetic data generation.
\newblock {\em IEEE Transactions on Privacy}, 2024.

\bibitem[JKT12]{jain2012differentially}
Prateek Jain, Pravesh Kothari, and Abhradeep Thakurta.
\newblock Differentially private online learning.
\newblock In {\em Conference on Learning Theory}, pages 24--1. JMLR Workshop and Conference Proceedings, 2012.

\bibitem[JLZ{\etalchar{+}}21]{jiang2021location}
Hongbo Jiang, Jie Li, Ping Zhao, Fanzi Zeng, Zhu Xiao, and Arun Iyengar.
\newblock Location privacy-preserving mechanisms in location-based services: A comprehensive survey.
\newblock {\em ACM Computing Surveys (CSUR)}, 54(1):1--36, 2021.

\bibitem[JYVDS18]{jordon2018pate}
James Jordon, Jinsung Yoon, and Mihaela Van Der~Schaar.
\newblock Pate-gan: Generating synthetic data with differential privacy guarantees.
\newblock In {\em International conference on learning representations}, 2018.

\bibitem[LC10]{lecun-mnisthandwrittendigit-2010}
Yann LeCun and Corinna Cortes.
\newblock {MNIST} handwritten digit database.
\newblock 2010.

\bibitem[LZKS24]{lavenant2024traj}
Hugo Lavenant, Stephen Zhang, Young-Heon Kim, and Geoffrey Schiebinger.
\newblock {Toward a mathematical theory of trajectory inference}.
\newblock {\em The Annals of Applied Probability}, 34(1A):428 -- 500, 2024.

\bibitem[Lé12]{léonard2010schrodinger}
Christian Léonard.
\newblock From the {S}chr\"odinger problem to the {M}onge–{K}antorovich problem.
\newblock {\em Journal of Functional Analysis}, 262(4):1879--1920, 2012.

\bibitem[Lé13]{léonard2013survey}
Christian Léonard.
\newblock A survey of the {S}chr\"odinger problem and some of its connections with optimal transport, 2013.

\bibitem[NCX21]{nguyen2021differentially}
Huy~L Nguyen, Anamay Chaturvedi, and Eric~Z Xu.
\newblock Differentially private k-means via exponential mechanism and max cover.
\newblock In {\em Proceedings of the AAAI conference on artificial intelligence}, volume~35, pages 9101--9108, 2021.

\bibitem[Nit24]{nitanda2024improved}
Atsushi Nitanda.
\newblock Improved particle approximation error for mean field neural networks.
\newblock In {\em Advances in Neural Information Processing Systems}, volume~37, pages 113823--113845, 2024.

\bibitem[NWS22]{nitanda2022convex}
Atsushi Nitanda, Denny Wu, and Taiji Suzuki.
\newblock Convex analysis of the mean field langevin dynamics.
\newblock In {\em International Conference on Artificial Intelligence and Statistics}, pages 9741--9757. PMLR, 2022.

\bibitem[PC{\etalchar{+}}19]{peyre2019computational}
Gabriel Peyr{\'e}, Marco Cuturi, et~al.
\newblock Computational optimal transport: With applications to data science.
\newblock {\em Foundations and Trends{\textregistered} in Machine Learning}, 11(5-6):355--607, 2019.

\bibitem[RSP{\etalchar{+}}20]{rogers2020linkedin}
Ryan Rogers, Subbu Subramaniam, Sean Peng, David Durfee, Seunghyun Lee, Santosh~Kumar Kancha, Shraddha Sahay, and Parvez Ahammad.
\newblock Linkedin's audience engagements {API}: A privacy preserving data analytics system at scale.
\newblock {\em arXiv preprint arXiv:2002.05839}, 2020.

\bibitem[SBB24]{shen2024multi}
Yunyi Shen, Renato Berlinghieri, and Tamara Broderick.
\newblock Multi-marginal schr$\backslash$" odinger bridges with iterative reference refinement.
\newblock {\em arXiv preprint arXiv:2408.06277}, 2024.

\bibitem[SCTS19]{saelens2019}
Wouter Saelens, Robrecht Cannoodt, Helena Todorov, and Yvan Saeys.
\newblock A comparison of single-cell trajectory inference methods.
\newblock {\em Nature Biotechnology}, 37(5):547--554, May 2019.

\bibitem[SGZ{\etalchar{+}}16]{salimans2016improved}
Tim Salimans, Ian Goodfellow, Wojciech Zaremba, Vicki Cheung, Alec Radford, and Xi~Chen.
\newblock Improved techniques for training gans.
\newblock volume~29, 2016.

\bibitem[SST{\etalchar{+}}19]{schiebinger2019reprogramming}
Geoffrey Schiebinger, Jian Shu, Marcin Tabaka, Brian Cleary, Vidya Subramanian, Aryeh Solomon, Joshua Gould, Siyan Liu, Stacie Lin, Peter Berube, Lia Lee, Jenny Chen, Justin Brumbaugh, Philippe Rigollet, Konrad Hochedlinger, Rudolf Jaenisch, Aviv Regev, and Eric~S. Lander.
\newblock Optimal-transport analysis of single-cell gene expression identifies developmental trajectories in reprogramming.
\newblock {\em Cell}, 176(4):928--943.e22, 2019.

\bibitem[SWN23]{suzuki2023convergencemeanfieldlangevindynamics}
Taiji Suzuki, Denny Wu, and Atsushi Nitanda.
\newblock Convergence of mean-field langevin dynamics: Time and space discretization, stochastic gradient, and variance reduction, 2023.

\bibitem[TKP19]{torkzadehmahani2019dp}
Reihaneh Torkzadehmahani, Peter Kairouz, and Benedict Paten.
\newblock Dp-cgan: Differentially private synthetic data and label generation.
\newblock In {\em Proceedings of the IEEE/CVF Conference on Computer Vision and Pattern Recognition Workshops}, pages 0--0, 2019.

\bibitem[TMF{\etalchar{+}}23]{tong2023simulation}
Alexander Tong, Nikolay Malkin, Kilian Fatras, Lazar Atanackovic, Yanlei Zhang, Guillaume Huguet, Guy Wolf, and Yoshua Bengio.
\newblock Simulation-free schr$\backslash$" odinger bridges via score and flow matching.
\newblock {\em arXiv preprint arXiv:2307.03672}, 2023.

\bibitem[TUV12]{thaler2012faster}
Justin Thaler, Jonathan Ullman, and Salil Vadhan.
\newblock Faster algorithms for privately releasing marginals.
\newblock In {\em International Colloquium on Automata, Languages, and Programming}, pages 810--821. Springer, 2012.

\bibitem[UV11]{ullman2011pcps}
Jonathan Ullman and Salil Vadhan.
\newblock Pcps and the hardness of generating private synthetic data.
\newblock In {\em Theory of Cryptography Conference}, pages 400--416. Springer, 2011.

\bibitem[VFG{\etalchar{+}}24]{ventre2024trajectoryinferencebranchingsde}
Elias Ventre, Aden Forrow, Nitya Gadhiwala, Parijat Chakraborty, Omer Angel, and Geoffrey Schiebinger.
\newblock Trajectory inference for a branching {SDE} model of cell differentiation, 2024.

\bibitem[VW19]{vempalawibisono}
Santosh Vempala and Andre Wibisono.
\newblock Rapid convergence of the unadjusted langevin algorithm: Isoperimetry suffices.
\newblock In {\em Advances in Neural Information Processing Systems}, volume~32. Curran Associates, Inc., 2019.

\bibitem[WWT{\etalchar{+}}18]{weinreb2018limits}
Caleb {Weinreb}, Samuel {Wolock}, Betsabeh~K. {Tusi}, Merav {Socolovsky}, and Allon~M. {Klein}.
\newblock {Fundamental limits on dynamic inference from single-cell snapshots}.
\newblock {\em Proceedings of the National Academy of Science}, 115(10):E2467--E2476, March 2018.

\bibitem[WZW{\etalchar{+}}23]{wang2023privtrace}
Haiming Wang, Zhikun Zhang, Tianhao Wang, Shibo He, Michael Backes, Jiming Chen, and Yang Zhang.
\newblock $\{$PrivTrace$\}$: Differentially private trajectory synthesis by adaptive markov models.
\newblock In {\em 32nd USENIX Security Symposium (USENIX Security 23)}, pages 1649--1666, 2023.

\bibitem[XLW{\etalchar{+}}18]{xie2018differentially}
Liyang Xie, Kaixiang Lin, Shu Wang, Fei Wang, and Jiayu Zhou.
\newblock Differentially private generative adversarial network.
\newblock {\em arXiv preprint arXiv:1802.06739}, 2018.

\bibitem[YCS{\etalchar{+}}22]{yang2022collecting}
Jianyu Yang, Xiang Cheng, Sen Su, Huizhong Sun, and Changju Chen.
\newblock Collecting individual trajectories under local differential privacy.
\newblock In {\em 2022 23rd IEEE International Conference on Mobile Data Management (MDM)}, pages 99--108. IEEE, 2022.

\bibitem[YNCY25]{yao2025learning}
Rentian Yao, Atsushi Nitanda, Xiaohui Chen, and Yun Yang.
\newblock Learning density evolution from snapshot data.
\newblock {\em arXiv preprint arXiv:2502.17738}, 2025.

\bibitem[ZLZ25]{zhang2025learning}
Zhenyi Zhang, Tiejun Li, and Peijie Zhou.
\newblock Learning stochastic dynamics from snapshots through regularized unbalanced optimal transport.
\newblock In {\em The Thirteenth International Conference on Learning Representations}, 2025.

\bibitem[ZW19a]{autodp1}
Yuqing Zhu and Yu-Xiang Wang.
\newblock Poission subsampled rényi differential privacy.
\newblock In Kamalika Chaudhuri and Ruslan Salakhutdinov, editors, {\em Proceedings of the 36th International Conference on Machine Learning}, volume~97 of {\em Proceedings of Machine Learning Research}, pages 7634--7642. PMLR, 09--15 Jun 2019.

\bibitem[ZW19b]{autodp2}
Yuqing Zhu and Yu-Xiang Wang.
\newblock Poission subsampled rényi differential privacy.
\newblock In Kamalika Chaudhuri and Ruslan Salakhutdinov, editors, {\em Proceedings of the 36th International Conference on Machine Learning}, volume~97 of {\em Proceedings of Machine Learning Research}, pages 7634--7642. PMLR, 09--15 Jun 2019.

\bibitem[ZW20]{autodp3}
Yuqing Zhu and Yu-Xiang Wang.
\newblock Improving sparse vector technique with renyi differential privacy.
\newblock In H.~Larochelle, M.~Ranzato, R.~Hadsell, M.F. Balcan, and H.~Lin, editors, {\em Advances in Neural Information Processing Systems}, volume~33, pages 20249--20258, 2020.

\end{thebibliography}

\appendix
\section{Additional results on differential privacy}\label{app:dp}
We include some additional results from the main text on differential privacy in this section.
\begin{definition}[$\ell_p$-sensitivity]
We define the $\ell_p$-sensitivity of a function $f$ to be \[
\Delta_p f:= \max_{\mathcal{D}, \mathcal{D}'}\|f(\mathcal{D}) - f(\mathcal{D}')\|_p,
\]
where $\mathcal{D}, \mathcal{D}'$ are adjacent datasets. %\mb{Don't use the same symbol $\mathcal{A}$ here as you do for an algorithm in the previous definition; in this definition, $\mathcal{A}$ stands for a deterministic function rather than a randomized algorithm obtained by adding noise to a deterministic function.}
\end{definition}
\begin{lemma}[Gaussian mechanism]\label{lem:gaussian_mechanism}
    Let $f$ be a function, $\epsilon,\delta \in (0, 1)$, and $\sigma^2 \ge \Delta_2f\frac{2\ln(1.25/\delta)}{\epsilon}$. The Gaussian mechanism $f(\mathcal{D})+\mathcal{N}(0, \sigma^2)$ is $(\epsilon,\delta)$-DP.
\end{lemma}

\begin{lemma}[Basic composition]\label{lem:basic_composition}
        Let $\mathcal{A}_1,\dots \mathcal{A}_k$ be $(\epsilon_i,\delta_i)$-DP algorithms. Then their composition $(\mathcal{A}_1(\mathcal{D}),\dots, \mathcal{A}_k(\mathcal{D}))$ is $(\sum_i \epsilon_i, \sum_i \delta_i)$-DP.
\end{lemma}

\begin{lemma}[Parallel composition]\label{lem:parallel_composition}
    Let $\mathcal{A}_1,\dots, \mathcal{A}_k$ be $(\epsilon,\delta)$-DP algorithms. Suppose $\mathcal{D} = S_1\cup \cdots \cup S_k$, where $S_i \cap S_j = \emptyset$ for every $i \neq j$. Then $(\mathcal{A}_1(S_1),\dots, \mathcal{A}_k(S_k))$ is $(\epsilon,\delta)$-DP.      
\end{lemma}

\section{Entropic optimal transport}\label{app:eot}

\begin{algorithm}[t]
\begin{algorithmic}[1]
\footnotesize
\Require{Empirical probability measures $\mu = \frac{1}{n}\sum \delta_{X_i},\nu = \frac{1}{m}\sum\delta_{Y_j}$, cost matrix $C\in \mathbb{R}^{n\times m}$, regularization parameter $\epsilon$, number of iterations $N$}
\State{$\varphi^{(0)} \gets \mathbf{1}_n$}
\State{$K \gets \exp(-C/\epsilon)$}
\For{$i = 1, \dots, N$}
\State{$\psi^{(i)}\gets \nu \odot K^\top \varphi^{(i-1)}$}
\State{$\varphi^{(i)}\gets \mu \odot K \psi^{(i-1)}$}

\EndFor
\State{\Return entropic optimal transport plan $\operatorname{diag}(\varphi^{(N)})K \operatorname{diag}(\psi^{(N)})$}
\end{algorithmic}
\caption{Sinkhorn}
\label{alg:sinkhorn}
\end{algorithm}

We provide a brief exposition to entropic optimal transport. See \citep{peyre2019computational} for a more thorough introduction. Let $\mathcal{X},\mathcal{Y}$ be arbitrary Polish spaces, $c: \mathcal{X}\times\mathcal{Y}\to\mathbb{R}$ be a cost function, and $\mu,\nu$ be probability measures on $\mathcal{X},\mathcal{Y}$, respectively. The entropic OT problem is
\[
T_\epsilon(\mu,\nu) =
\inf_{\pi\in \Pi(\mu,\nu)}\int_{\mathcal{X}\times\mathcal{Y}}c(x,y)\pi(dx,dy) + \epsilon H(\pi|\mu\otimes \nu),
\]
where $\Pi(\mu,\nu)$ is the set of all probability measures on $\mathcal{X}\times\mathcal{Y}$ with marginals $\mu$ on $\mathcal{X}$ and $\nu$ on $\mathcal{Y}$, $H$ is the relative entropy, and $\epsilon$ is the regularization parameter. By duality theory, this is equivalent to the following problem \[
T_\epsilon(\mu,\nu) = \max_{\varphi\in L^1(\mu), \psi\in L^1(\nu)} \int\varphi d\mu + \int\psi d\nu + \epsilon\left(1 - \int e^{\frac{1}{\epsilon}(\varphi(x) + \psi(y) - c(x,y))}d\mu(x)d\mu(y)\right),
\]
which admits a unique solution up to a translation $(\varphi + \kappa,\psi - \kappa)$ for $\kappa\in \mathbb{R}$. Furthermore, the functions $(\varphi,\psi)$ satisfy the following conditions: \[
\begin{cases}
    \varphi(x) = -\epsilon \log \int \exp(\frac{1}{\epsilon}(\varphi(y) - c(x,y)))d\nu(y)\\
    \psi(y) = -\epsilon \log \int \exp(\frac{1}{\epsilon}(\varphi(y) - c(x,y)))d\mu(x).
\end{cases}
\]
In the discrete (empirical measure) setting, these potentials give motivation for the Sinkhorn algorithm, which we describe in Algorithm \ref{alg:sinkhorn}, where $\odot$ is taken to be element-wise multiplication. 

Remarkably, the entropic optimal transport problem has been shown to be equivalent to the Schr\"dinger bridge problem \cite{léonard2010schrodinger,léonard2013survey}, which inspires the trajectory inference algorithms from \cite{chizat2022trajectoryinferencemeanfieldlangevin,gu2025partially}.

\section{Additional results on mean-field Langevin dynamics}\label{app:mfld}

\subsection{Representer theorem}\label{app:rep_thm}
See \cite{chizat2022trajectoryinferencemeanfieldlangevin} for the proof of the following.
\begin{theorem}[Representer theorem]
    Let $\mathrm{Fit}:\mathcal{P}(\mathcal{X})^{T}\to \mathbb{R}$ be \emph{any} function. The following hold.
    \begin{enumerate}
        \item[(i)] If $\mathcal{F}$ admits a minimizer $\mathbf{R}^*$ then $(\mathbf{R}_{t_1}^*, \dots, \mathbf{R}_{t_T}^*)$ is a minimizer for $F$. 
        \item[(ii)] If $F$ admits a minimizer $\boldsymbol{\mu}^*\in\mathcal{P}(\mathcal{X})^{T}$, then a minimizer $\mathbf{R}^*$ for $\mathcal{F}$ is built as \[
        \mathbf{R}^*(\cdot) = \int_{\mathcal{X}^{T}}\mathbf{W}^{\tau}(\cdot|x_1,\dots, x_T)\,d\mathbf{R}_{t_1,\dots, t_T}(x_1,\dots, x_T),
        \]
        where $\mathbf{W}^{\tau}(\cdot|x_1,\dots, x_T)$ is the law of $\mathbf{W}^{\tau}$ conditioned on passing through $x_1,\dots, x_T$ at times $t_1,\dots, t_T$, respectively  and $\mathbf{R}_{t_1,\dots, t_T}$ is the composition of the optimal transport plans $\gamma_{i}$ that minimize $T_{\tau'}(\boldsymbol{\mu}^{*(i)}, \boldsymbol{\mu}^{*(i+1)})$, for $i \in [T]$.
    \end{enumerate}

    \label{thm:representer}
 \end{theorem}
The composition of the transport plans is obtained as: 
\begin{equation} \label{eq:transport_compn}
\mathbf{R}_{t_1,\dots,t_T}(dx_1,\dots, dx_T) = \gamma_{1}(dx_1,dx_2)\gamma_{2}(dx_3|x_2)\cdots \gamma_{T-1}(dx_T|x_{T-1}),
\end{equation}
where the OT plans $\gamma_{i}(dx_i,dx_{i+1}) = \gamma_{i}(dx_{i+1}|x_i)\mu_i(dx_i)$ are conditional probabilities (or ``disintegrations''). The ``reduction'' of the optimization space from $\mathcal{P}(\Omega)$ to $\mathcal{P}(\mathcal{X})^{T}$ is enabled by the Markov property of $\mathbf{W}^{\tau}$. Thus, Theorem \ref{thm:representer} allows us to compute a minimizer for $\mathcal{F}$ from a minimizer for $F$ and its associated OT plans.

\subsection{Properties of $G$ and $F$}\label{app:properties_of_functional}
The \emph{first-variation} of $G: \mathcal{P}(\mathcal{X})^{T}\to\mathbb{R}$ is the unique (up to an additive constant) function $V[\boldsymbol{\mu}]\in \mathcal{C}(\mathcal{X})^{T}$ such that for all $\boldsymbol{\nu}\in\mathcal{P}(\mathcal{X})^{T}$, it holds \[
\lim_{\epsilon\downarrow 0}\frac{1}{\epsilon}[G((1-\epsilon)\boldsymbol{\mu}+\epsilon\boldsymbol{\nu})-G(\boldsymbol{\mu})] = \sum_{i=1}^T \int V^{(i)}[\boldsymbol{\mu}]d(\boldsymbol{\nu}-\boldsymbol{\mu})^{(i)}.
\]
We have the following properties of $G$ and $V$.
\begin{proposition}\label{prop:propertiesG_F}
    The function $G$ is convex separately in each of its input (but not jointly), weakly continuous, and its first-variation is given for $\boldsymbol{\mu}\in\mathcal{P}(\mathcal{X})^{T}$ and $i\in[T]$ by \[
    V^{(i)}[\boldsymbol{\mu}] = \frac{\delta\mathrm{Fit}}{\delta\boldsymbol{\mu}^{(i)}}[\boldsymbol{\mu}] + \frac{\varphi_{i,i+1}}{\Delta t_i} + \frac{\psi_{i,i-1}}{\Delta t_{i-1}}, \quad \frac{\delta \mathrm{Fit}}{\delta \boldsymbol{\mu}^{(i)}}[\boldsymbol{\mu}] : x\mapsto - \frac{\Delta t_i}{\lambda}\int \frac{g_\sigma(x-y)}{(g_\sigma * \boldsymbol{\mu}^{(i)})(y)}d\hat{\mu}_{t_i}(y),
    \]
    where $(\varphi_{i,j},\psi_{i,j})\in\mathcal{C}^\infty(\mathcal{X})$ are the Schr\"odinger potentials\footnote{See Appendix \ref{app:eot}.} for $T_{\tau'}(\boldsymbol{\mu}^{(i)},\boldsymbol{\mu}^{(j)})$, with the convention that the corresponding term vanishes when it involves $\psi_{1,0}$ or $\varphi_{T,T+1}$. The function $F$ is jointly convex and admits a unique minimizer $\boldsymbol{\mu}^{*}$, which has an absolutely continuous density characterized by \[
    (\boldsymbol{\mu}^*)^{(i)}\propto e^{-V^{(i)}[\boldsymbol{\mu}^*]/\tau},
    \]
    for $i\in [T]$.
\end{proposition}
\begin{proof}
    See \cite{chizat2022trajectoryinferencemeanfieldlangevin}.
\end{proof}

% \subsection{Fokker-Planck PDE representation of \eqref{eq:mckean_vlasov}}

\subsection{Convergence of MFLD}
For convergence of the MFLD \cite{nitanda2022convex,chizat2022meanfieldlangevindynamicsexponential}, we start with the assumption on the (log-Sobolev inequality) LSI constant.
\begin{assumption}\label{asmp:lsi}
    There exists a constant $\CLSI > 0$ such that $\boldsymbol{\mu}^*$ satisfies a log-Sobolev inequality with constant $\CLSI$.
\end{assumption}

\cite{chizat2022trajectoryinferencemeanfieldlangevin,gu2025partially} utilize the Holley-Stroock perturbation argument \cite{holleystroock} to bound the LSI constant. Under this assumption, they show the exponential convergence in continuous-time and the infinite-particle limit, e.g. the number of particles is $m \to \infty$ in each distribution. 
\begin{theorem}[Convergence]\label{thm:exponential_convergence}
   Assume Assumption \ref{asmp:lsi}. Let $\boldsymbol{\mu}_0\in\mathcal{P}(\mathcal{X})^{T}$ be such that $F(\boldsymbol{\mu}_0)<+\infty$. Then for $\tau > 0$, there exists a unique solution $(\boldsymbol{\mu}_t)_{t\ge0}$ to \eqref{eq:mckean_vlasov}, and further it holds 
    $F_\tau(\boldsymbol{\mu}_t) - \min F_\tau \le \exp(-2\CLSI \tau t)(F_\tau(\boldsymbol{\mu}_0)-\min F_\tau)$.
\end{theorem}

\subsection{Remarks on discretization}
 We make some remarks on the discretization error, which for instance has been shown for mean-field neural networks in \cite{suzuki2023convergencemeanfieldlangevindynamics,nitanda2024improved,chewi2024uniform}. These works analyze the setting where the functional has $L_2$ or strongly convex regularization, which is beneficial for mean-field neural networks but undesirable in our setting. It is likely the discretized analysis will still hold without this regularization, but we defer this to future work on statistical guarantees for MFLD, and discuss the discretization error in the strongly-convex regularization setting. Furthermore, we remark that the fully discretized analysis in our setting will also be complicated by the fact that we need to take into account the estimation error of entropic OT maps (and thus the Schr\"odinger potentials) as well.

The convergence rate in \emph{discrete-time} is exponential under a uniform-in-$m$ LSI constant \cite{nitanda2024improved,chewi2024uniform}, utilizing the classical one-step interpolation argument \cite{vempalawibisono,suzuki2023convergencemeanfieldlangevindynamics}. The stochastic gradient incurs an $O\left(\frac{1}{\rho N}\right)$ error \cite{suzuki2023convergencemeanfieldlangevindynamics}. By uniform-in-time propagation of chaos \cite{chewi2024uniform}, the error between the discretization and the mean-field limit is $O\left(\frac{1}{m}\right)$ uniformly in time (over the optimization dynamics).

\section{Proofs}\label{app:proofs}
\begin{proposition}[Prop. {\ref{prop:rand_subset}}, restated]
    Let $\Tilde{T}\subset [T-1]$ be a random subset of size $z = \Omega(\log^2T)$. Then $\mathbb{E}[\max(\Tilde{T}_{i} - \Tilde{T}_{i-1})]
     = O\left(\frac{T}{z}\log z\right)$.
\end{proposition}
\begin{proof}
    Choose $\tilde T=\{t_1 < \cdots < t_z\}$ uniformly from $[T-1]={1,\dots,T-1}$, with $z\ge1$ and add the endpoints $T_0 = 0$ and $T_{z+1}=T$. Define the random variables $G_i := t_{i} - t_{i-1}$ for $i \in [z+1]$. We note that 
    \begin{equation}\label{eq:a1}
        \Pr[\max G_i \ge t] \le \Pr[\exists \text{ a block of }t \text{ consecutive integers disjoint from }\tilde T].
    \end{equation}    
    Consider a consecutive block $B$ of $t$ points. We can bound the probability that no point in $\tilde T$ is in $B$ as follows:  \begin{align}
        \Pr[B \cap \tilde T = \emptyset] &= \frac{{T-t\choose z}}{{T\choose z}}\notag \\
        &= \prod_{i=1}^{z-1}\frac{T - t - i}{T - i}\notag \\
        &\le \left(1 - \frac{t}{T}\right)^z\notag \\
        &\le \exp\left(- \frac{zt}{T}\right).\label{eq:a2}
    \end{align}
    Using \eqref{eq:a1} and \eqref{eq:a2}, by a union bound, we have \begin{align}
        \Pr[\max G_i \ge t] \le T  \exp\left(- \frac{zt}{T}\right).
    \end{align}
    Now we have \begin{align*}
        \mathbb{E}[\max G_i] &= \sum_{t \in \mathbb{N}}\Pr[\max G_i\ge t]\\
        &= \sum_{t\le \frac{T}{z}\log z}\Pr[\max G_i \ge t] + \sum_{t > \frac{T}{z}\log z}\Pr[\max G_i \ge t]\\
        &\le \sum_{t\le \frac{T}{z}\log z}1 + \sum_{t>\frac{T}{z}\log z}\Pr[\max G_i\ge t]\\
        &\le \frac{T}{z}\log z + \sum_{t> \frac{T}{z}\log z} T\exp\left(- \frac{zt}{T}\right)\\
        &\le \frac{T}{z}\log z + T \cdot \frac{\exp\left(-z(\frac{T}{z}\log z))/T\right)}{1 - \exp(-z/T)}\\
        &\le \frac{T}{z}\log z + \frac{T}{z} \cdot \frac{1}{1 - \exp\left(-z/T\right)},
    \end{align*}
    from which the claim follows so long as $z \ge \log^2T$.
\end{proof}

% \section{Remarks on improving runtime}
% Due to the runtime of entropic OT, one iteration will take time $O(Tm^2)$, whereas the ideal runtime for the mean-field Langevin algorithm should be $O(Tm)$ per iteration. To alleviate this, it is possible to ``parallelize'' the entropic OT between two consective distributions by running $\sqrt{m}$ instances of entropic OT between disjoint sets of $\sqrt{m}\times \sqrt{m}$ problems, which we obtain by taking permutations of consecutive time-marginal distributions. This reduces the runtime of one iteration to $O(Tm^{3/2})$. Note that this can easily be run in parallel, which would imply a runtime of $O(Tm)$ (given sufficient threads). Using parallelization of entropic OT incurs $O(m^{-1/2})$ error due to the sample complexity of entropic OT \cite{genevay2019sample}. 

\section{Additional experiments}\label{app:experiments}

\begin{figure}
    \centering
    \includegraphics[width=0.45\textwidth]{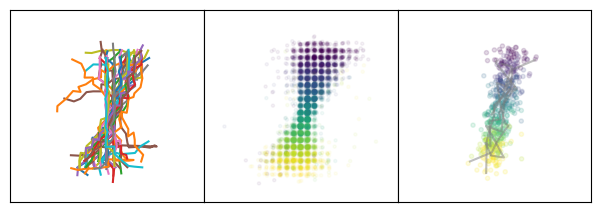}~\includegraphics[width=0.45\textwidth]{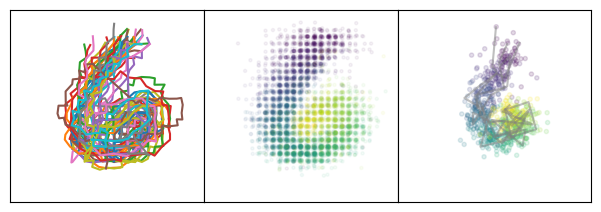}
    \includegraphics[width=0.45\textwidth]{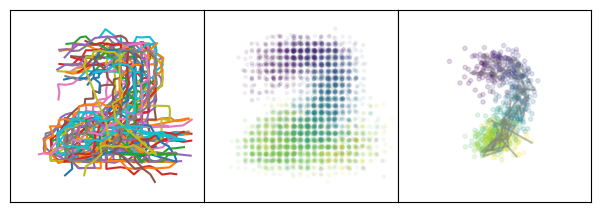}~\includegraphics[width=0.45\textwidth]{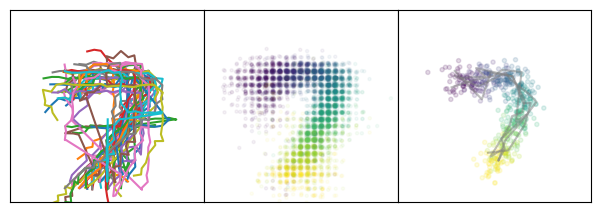}
    \includegraphics[width=0.45\textwidth]{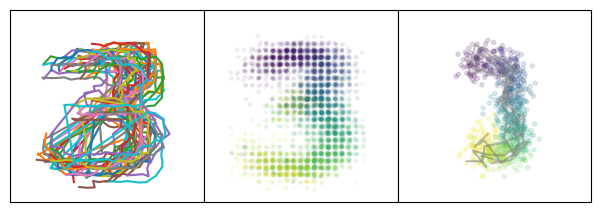}~\includegraphics[width=0.45\textwidth]{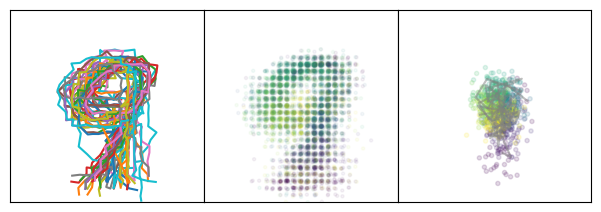}
    \caption{MNIST data.}
    \label{fig:main_all}
\end{figure}

\begin{table}
\centering
\begin{tabular}{|c|c|c|c|}
\hline
Digit & $N$ & $T$ & $\epsilon(\delta=10^{-4})$\\
\hline
   1 & 674 & 10 & 1.19 \\
\hline
  2 & 397 & 15 & 1.98 \\
\hline
3 & 408 & 15 & 1.92 \\
\hline
   6  & 394 & 15 & 1.05 \\
\hline
7 & 626 & 10 & 1.28 \\
\hline
9 & 396 & 15 & 1.98 \\
\hline
\end{tabular}
    \caption{Parameters corresponding to Figure \ref{fig:main_all}, with $\delta = 10^{-4}$.}
    \label{tab:main_all}
    \vspace{-0.7cm}
\end{table}

\begin{figure}[h]
    \centering
    \begin{subfigure}[t]{0.3\textwidth}
    \centering
    \includegraphics[width=0.8\textwidth]{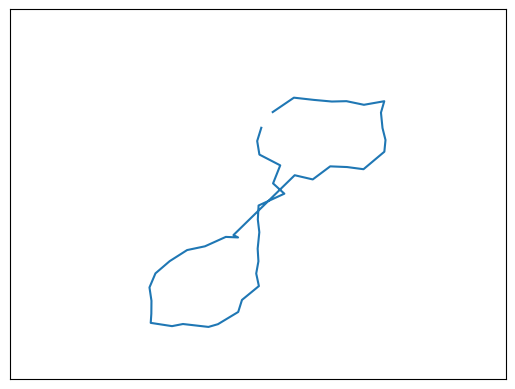}
    \caption{Classic 8.}
    \end{subfigure}~~~
    \begin{subfigure}[t]{0.3\textwidth}
    \centering
    \includegraphics[width=0.8\textwidth]{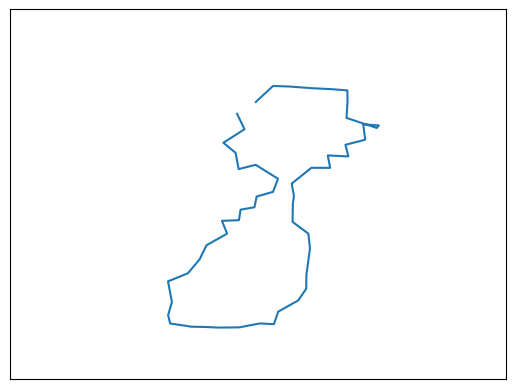}
    \caption{Incorrect 8.}
    \label{fig:incorrect8}
    \end{subfigure}
        \caption{Example of digit 8s.}
\end{figure}

We provide experiments without warm starts for MNIST in Figure \ref{fig:main_all}. Here, we use $\eta=0.01$, $\rho = \frac{5}{N}$, and initialization $\mathcal{N}((0.5, 0.5), 0.2\cdot I_2)$. We provide privacy budgets in Table \ref{tab:main_all}.

\end{document}